\documentclass[10pt,twocolumn,letterpaper]{article}

\usepackage{cvpr}              

\usepackage{graphicx}
\usepackage{amsmath} 
\usepackage{amssymb}
\usepackage{booktabs}
\usepackage{stfloats}

\usepackage[pagebackref,breaklinks,colorlinks]{hyperref}

\usepackage[capitalize]{cleveref}
\crefname{section}{Sec.}{Secs.}
\Crefname{section}{Section}{Sections}
\Crefname{table}{Table}{Tables}
\crefname{table}{Tab.}{Tabs.}

\usepackage{amsmath}
\usepackage{algorithm}
\usepackage{algpseudocode}  
\usepackage{amsfonts}
\usepackage{enumerate}
\usepackage[accsupp]{axessibility}  

\usepackage{ragged2e}
\usepackage{booktabs}
\usepackage{color}
\usepackage{multirow}
\usepackage{bm}
\usepackage{bbm}
\usepackage{graphicx}
\usepackage{caption}
\usepackage[labelformat=simple]{subcaption}

\newtheorem{theorem}{Theorem}[subsection]

\newtheorem{lemma}[theorem]{Lemma}
\newtheorem{corollary}[theorem]{Corollary}

\newtheorem{definition}[theorem]{Definition}

\newtheorem{proof}{Proof}[theorem]
\newtheorem{remark}{Remark}[theorem]
\newtheorem{property}{Property}[theorem]

\def\ie{{\em i.e.}}
\def\eg{{\em e.g.}}

\newcommand{\figref}[1]{Fig. \ref{#1}}
\newcommand{\tabref}[1]{Tab. \ref{#1}}

\newcommand{\myPara}[1]{\vspace{.05in}\noindent\textbf{#1}}

\newcommand{\mc}[1]{\mathcal{#1}}

\newcommand{\mb}[1]{\mathbb{#1}}

\newcommand{\tabincell}[2]{\begin{tabular}{@{}#1@{}}#2\end{tabular}}

\makeatletter
\def\thanks#1{\protected@xdef\@thanks{\@thanks
        \protect\footnotetext{#1}}}
\makeatother

\def\name{TriDet}
\def\modulename{SGP}

\def\cvprPaperID{6212} 
\def\confName{CVPR}
\def\confYear{2023}

\begin{document}

\title{TriDet: Temporal Action Detection with Relative Boundary Modeling}


\author{Dingfeng Shi$^{\ast}$\thanks{*: This work is done during an internship at JD Explore Academy.
}\\
VRLab, Beihang University, China\\
{\tt\small shidingfeng@buaa.edu.cn}
\and Yujie Zhong\\
Meituan Inc.\\
{\tt\small jaszhong@hotmail.com}
\and
Qiong Cao$^{\dag}$\thanks{$^{\dag}$: Corresponding authors.}\\
JD Explore Academy\\
{\tt\small mathqiong2012@gmail.com}\\
\and Lin Ma\\
Meituan Inc.\\
{\tt\small forest.linma@gmail.com}
\and Jia Li$^{\dag}$\\
VRLab, Beihang University, China\\
{\tt\small jiali@buaa.edu.cn}
\and Dacheng Tao\\
JD Explore Academy\\
{\tt\small dacheng.tao@gmail.com}\\
}

%

\maketitle

\begin{abstract}
In this paper, we present a one-stage framework TriDet for temporal action detection. 
Existing methods often suffer from imprecise boundary predictions due to the ambiguous action boundaries in videos. To alleviate this problem, we propose a novel Trident-head to model the action boundary via an estimated relative probability distribution around the boundary. In the feature pyramid of TriDet, we propose an efficient Scalable-Granularity Perception (SGP) layer to mitigate the rank loss problem of self-attention that takes place in the video features and aggregate information across different temporal granularities. 
Benefiting from the Trident-head and the SGP-based feature pyramid, TriDet achieves state-of-the-art performance on three challenging benchmarks: THUMOS14, HACS and EPIC-KITCHEN 100, with lower computational costs, compared to previous methods. 
For example, TriDet hits an average mAP of $69.3\%$ on THUMOS14, outperforming the previous best by $2.5\%$, but with only $74.6\%$ of its latency. The code is released to \href{https://github.com/dingfengshi/TriDet}{https://github.com/dingfengshi/TriDet}. 
\end{abstract}

\begin{figure}[t]
    \centering{
    \includegraphics[width=0.87\linewidth]{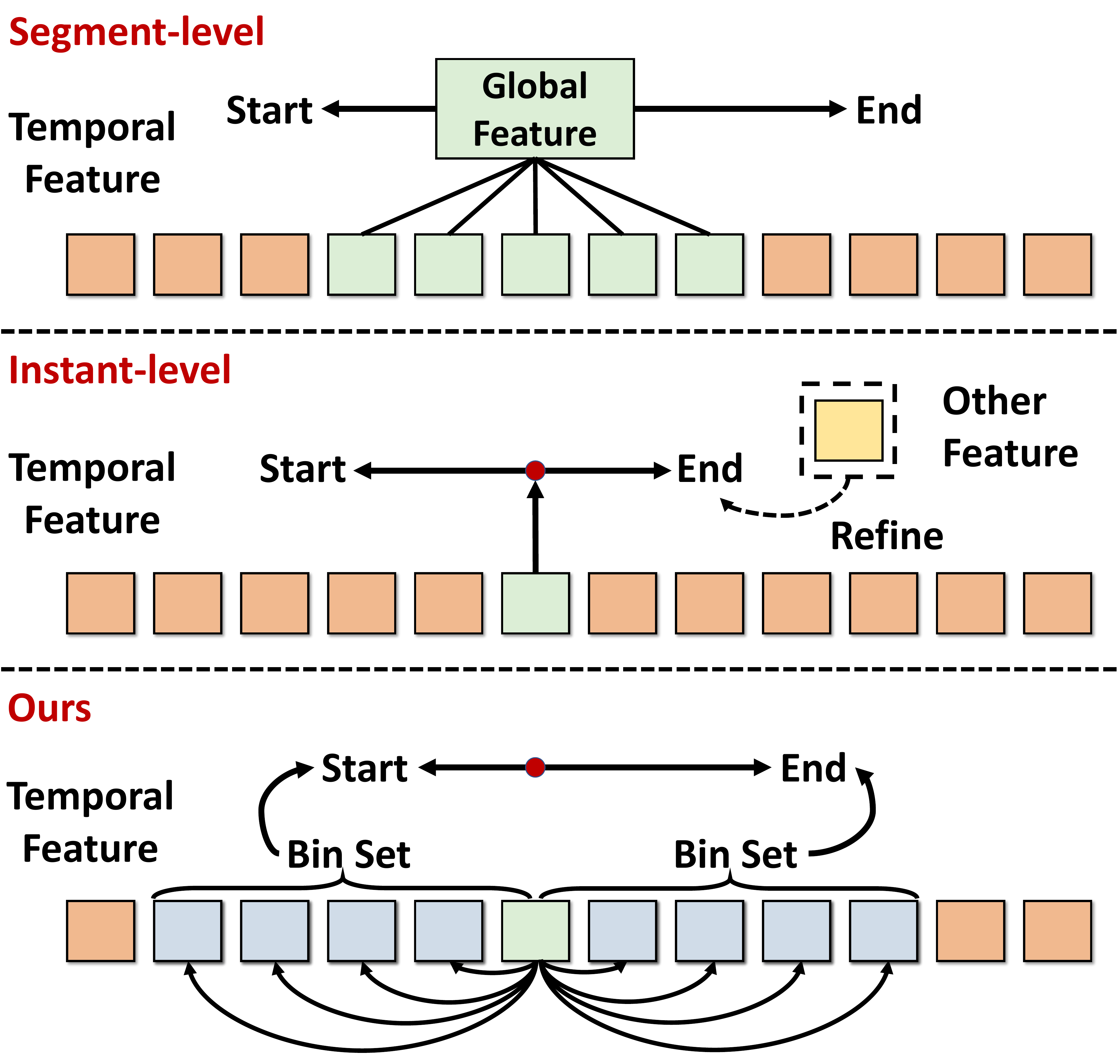}
    
  \caption{Illustration of different boundary modeling. \textbf{Segment-level}: these methods locate the boundaries based on the global feature of a predicted temporal segment. \textbf{Instant-level}: they directly regress the boundaries based on a single instant, potentially with some other features. \textbf{Ours}: the action boundaries are modeled via an estimated relative probability distribution of the boundary.}
  \label{fig:motivation} 
  }
\end{figure}

\section{Introduction}
\label{sec:intro}

Temporal action detection (TAD) aims to detect all start and end instants and corresponding action categories from an untrimmed video, which has received widespread attention. TAD has been significantly improved with the help of the deep learning. 
However, TAD remains to be a very challenging  task due to some unresolved problems.

A critical problem in TAD is that action boundaries are usually not obvious. 
Unlike the situation in object detection where there are usually clear boundaries between the objects and the background, the action boundaries in videos can be fuzzy. A concrete manifestation of this is that the instants (\ie~temporal locations in the video feature sequence) around the boundary have relatively higher predicted response value from the classifier. 


Some previous works attempt to locate the boundaries based on the global feature of a predicted temporal segment~\cite{lin2018bsn,lin2019bmn, zeng2019graph,zhao2020bottom,long2019gaussian},
which may ignore detailed information at each instant.
As another line of work, they directly regress the boundaries based on a single instant~\cite{zhang2022actionformer,paul2018w}, potentially with some other features~\cite{lin2021learning,qing2021temporal,zhao2021video}, which do not consider the relation between adjacent instants (\eg~the relative probability) around the boundary.  
How to effectively utilize boundary information remains an open question. 

\begin{figure}[t]
    \centering{
    \includegraphics[width=0.85\linewidth]{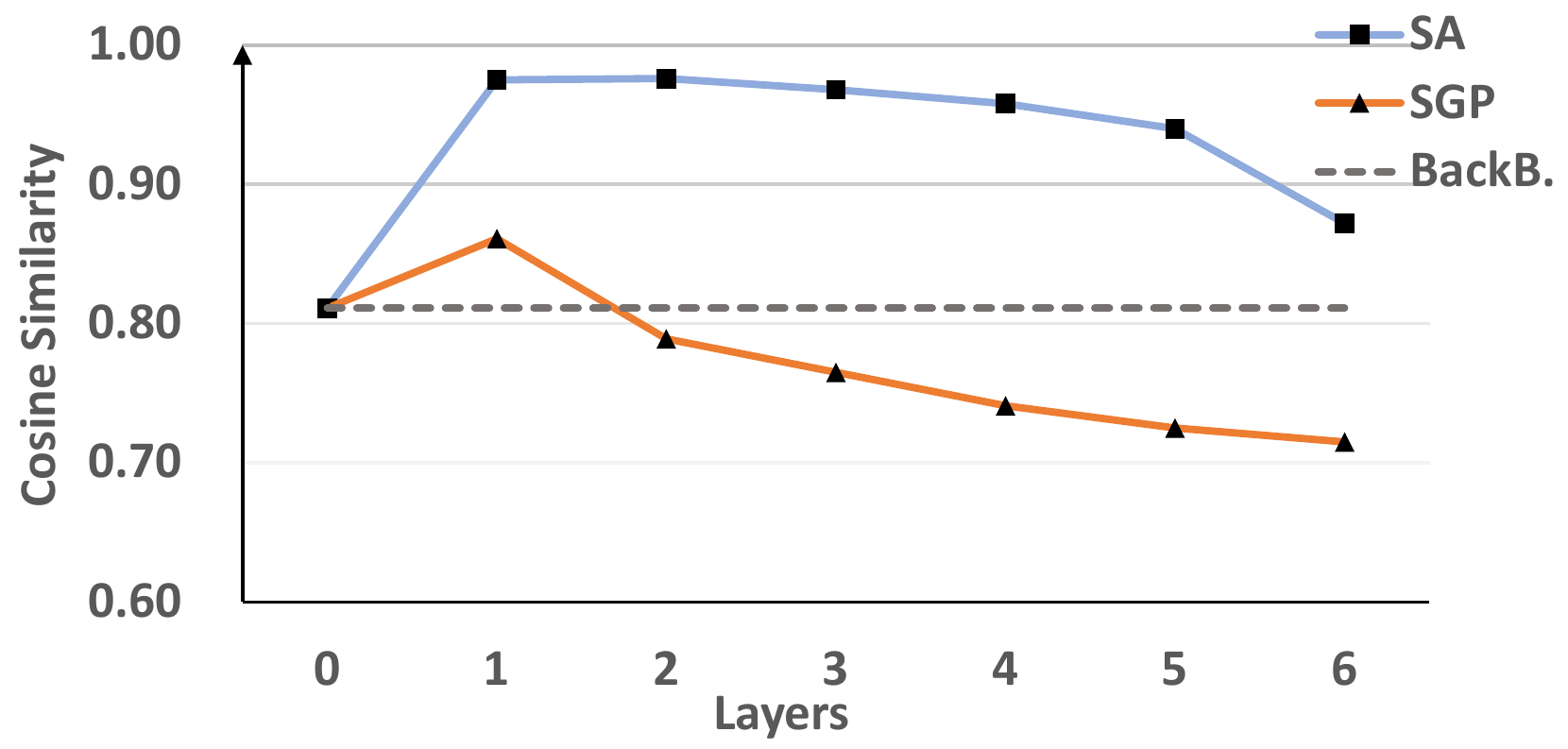}
    
  \caption{Within the HACS dataset and SlowFast backbone, we statistic the average cosine similarity between features at each instant and the video-level average feature for self-attention and SGP, respectively. We observe that the SA exhibits high similarity, indicating poor discriminability (\ie~rank loss problem). In contrast, SGP resolves the issue and exhibits stronger discriminability.
  }
  \label{fig:cosine} 
  }
\end{figure}

To facilitate localization learning, we posit that the relative response intensity of temporal features in a video can mitigate the impact of video feature complexity and increase localization accuracy. 
Motivated by this, we propose a one-stage action detector with a novel detection head named Trident-head tailored for action boundary localization. Specifically, instead of directly predicting the boundary offsets based on the center point feature, the proposed Trident-head models the action boundary via an estimated relative probability distribution of the boundary (see \figref{fig:motivation}). The boundary offset is then computed based on the expected values of neighboring locations (\ie~bins).

Apart from the Trident-head, in this work, the proposed action detector consists of a backbone network and a feature pyramid. 
Recent TAD methods~\cite{zhang2022actionformer,cheng2022tallformer,weng2022efficient} adopt the transformer-based feature pyramid and show promising performance. 
However, the video features of the video backbone tend to exhibit high similarities between snippets, 
which is further deteriorated by SA, leading to the rank loss problem~\cite{dong2021attention} (see \figref{fig:cosine}).
Additionally, SA also incurs significant computational overhead.

Fortunately, we discover that the success of the previous transformer-based layers (in TAD) primarily relies on their macro-architecture, namely, how the normalization layer and feed-forward network (FFN) are connected, rather than the self-attention mechanism.
We therefore propose an efficient convolutional-based layer, termed Scalable-Granularity Perception (SGP) layer, to alleviate the two 
abovementioned problems of self-attention.
SGP comprises two primary branches, which serve to increase the discrimination of features in each instant and capture temporal information with different scales of receptive fields.

The resultant action detector is termed \name. 
Extensive experiments demonstrate that \name~surpasses all the previous detectors and achieves state-of-the-art performance across three challenging benchmarks: THUMOS14, HACS and EPIC-KITCHEN 100. 


\begin{figure*}[t]
    \centering{
    \includegraphics[width=0.92\linewidth]{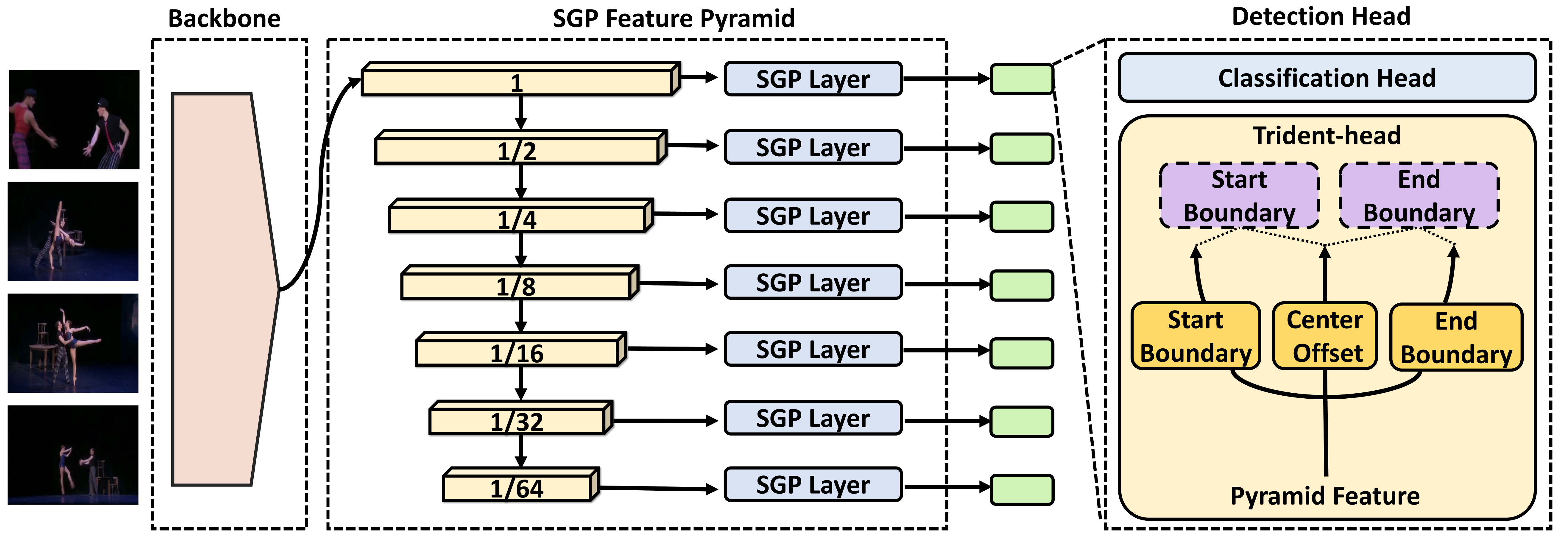}
    
  \caption{Illustration of \name. We build the pyramid features with Scalable-Granularity Perception (SGP) layer. The corresponding features in each level are fed into a shared-weight detection head to obtain the detection result, which consists of a classification head and a Trident-head. The Trident-head estimates the boundary offset based on a relative distribution predicted by three branches: Start Boundary, End Boundary and Center Offset.}
  \label{fig:framework} 
  }
\end{figure*}

\section{Related Work}
\label{sec:related}

\myPara{Temporal action detection.}
Temporal action detection (TAD) involves localizing and classifying all actions from an untrimmed video. The existing methods can be roughly divided into two categories, namely, two-stage methods and one-stage methods.
The two-stage methods~\cite{xu2020g,zeng2019graph,zhu2021enriching,sridhar2021class,  qing2021temporal} split the detection process into two stages: proposal generation and proposal classification. Most of the previous works~\cite{lin2018bsn,lin2019bmn,lin2020fast,chen2022dcan,escorcia2016daps,liu2021multi} put emphasis on the proposal generation phrase. Concretely, some works~\cite{lin2019bmn,lin2018bsn,chen2022dcan} predict the probability of the action boundary and densely match the start and end instants according to the prediction score. Anchor-based methods~\cite{lin2020fast,escorcia2016daps} classify actions from specific anchor windows. However, two-stage methods suffer from a high complexity problem and can not be trained in an end-to-end manner. 
The one-stage methods do the localization and classification with a single network. Some previous works~\cite{yang2020revisiting,lin2021learning,yang2022basictad} build this hierarchical architecture with the convolutional network (CNN).
However, there is still a performance gap between the CNN-based and the latest TAD methods.

\myPara{Object detection.}
Object detection is a twin task of TAD. General Focal Loss~\cite{li2020generalized} transforms bounding box regression from learning Dirac delta distribution to a general distribution function. Some methods~\cite{howard2017mobilenets,chollet2017xception,liu2022convnet} use Depth-wise Convolution to model network structure and some branched designs~\cite{szegedy2017inception,hu2018squeeze} show high generalization ability. They are enlightening for the architecture design of TAD.

\myPara{Transformer-based methods.}
Inspired by the great success of the Transformer in the field of machine translation and object detection, some recent  works~\cite{zhang2022actionformer, shi2022react, tan2021relaxed, cheng2022tallformer,liu2022end,liu2022empirical} adopt the attention mechanism in TAD task, which help improve the detection performance. For example, some works\cite{tan2021relaxed,shi2022react,liu2022end} detect the action with the DETR-like Transformer-based decoder~\cite{carion2020end}, which models action instances as a set of learnable. Other works~\cite{zhang2022actionformer,cheng2022tallformer} extract a video representation with a Transformer-based encoder. However, most of these methods are based on the \emph{local} behavior. Namely, they conduct attention operation only in a local window, which introduces an inductive bias similar to CNN but with a larger computational complexity and additional limitations (\eg~The length of the sequence needs to be pre-padded to an integer multiple of the window size.). 

\section{Method}
\myPara{Problem definition.} We first give a formal definition for TAD task. Specifically, given a set of untrimmed videos $\mc{D}=\{\mc{V}_i\}_{i=1}^{n}$, we have a set of RGB (and optical flow) temporal visual features $X_i=\{{x_t}\}_{t=1}^T$ from each video $\mc{V}_i$, where $T$ corresponds to the number of instants, and $K_i$ segment labels $Y_i=\{s_k,e_k,c_k\}_{k=1}^{K_i}$ with the action segment start instant $s_k$, the end instant $e_k$ and the corresponding action category $c_k$. TAD aims at detecting all segments $Y_i$ based on the input feature $X_i$. 

\subsection{Method Overview}
Our goal is to build a simple and efficient one-stage temporal action detector. 
As shown in \figref{fig:framework}, the overall architecture of \name~consists of three main parts: a video feature backbone, a SGP feature pyramid, and a boundary-oriented Trident-head. 
First, the video features are extracted using a pretrained action classification network (\eg{~I3D~\cite{carreira2017quo} or SlowFast~\cite{feichtenhofer2019slowfast}}). Following that, a SGP feature pyramid is built to tackle actions with various temporal lengths, similar to some recent TAD works~\cite{lin2021learning,zhang2022actionformer,cheng2022tallformer}.
Namely, the temporal features are iteratively downsampled and each scale level is processed with a proposed Scalable-Granularity Perception (\modulename) layer (Section~\ref{sec:efp}) to enhance the interaction between features with different temporal scopes. 
Lastly, action instances are detected by a designed boundary-oriented Trident-head (Section~\ref{sec:head}). 
We elaborate on the proposed modules in the following.

\subsection{Feature Pyramid with \modulename~Layer}
\label{sec:efp}

The feature pyramid is obtained by first downsampling the output features of the video backbone network several times via max-pooling (with a stride of 2). The features at each pyramid level are then processed using transformer-like layers (e.g. ActionFormer~\cite{zhang2022actionformer}). 

Current Transformer-based methods for TAD tasks primarily rely on the macro-architecture of the Transformer (See supplementary material for details), rather than the self-attention mechanisms. Specifically, SA mainly encounters two issues: the rank loss problem across the temporal dimension and its high computational overhead.

\paragraph{Limitation 1: the rank loss problem.}
The rank loss problem arises because the probability matrix in self-attention (\ie~softmax($QK^T$)) is \emph{non-negative} and \emph{the sum of each row is 1}, indicating the outputs of SA are \emph{convex combination} for the value feature $V$. Considering that pure Layer Normalization~\cite{ba2016layer} projects feature onto the unit hyper-sphere in high-dimensional space, we analyze the degree of their distinguishability by studying the maximum angle between features within the instant features. We demonstrate that the maximum angle between features after the \emph{convex combination} is less than or equal to that of the input features, resulting in increasing similarity between features (as outlined in the supplementary material), which can be detrimental to TAD.

\paragraph{Limitation 2: high computational complexity.}
In addition, the dense pair-wise calculation (between instant features) in self-attention brings a high computational overhead and therefore decreases the inference speed.

\begin{figure}[t]
    \centering{
    \includegraphics[width=0.87\linewidth]{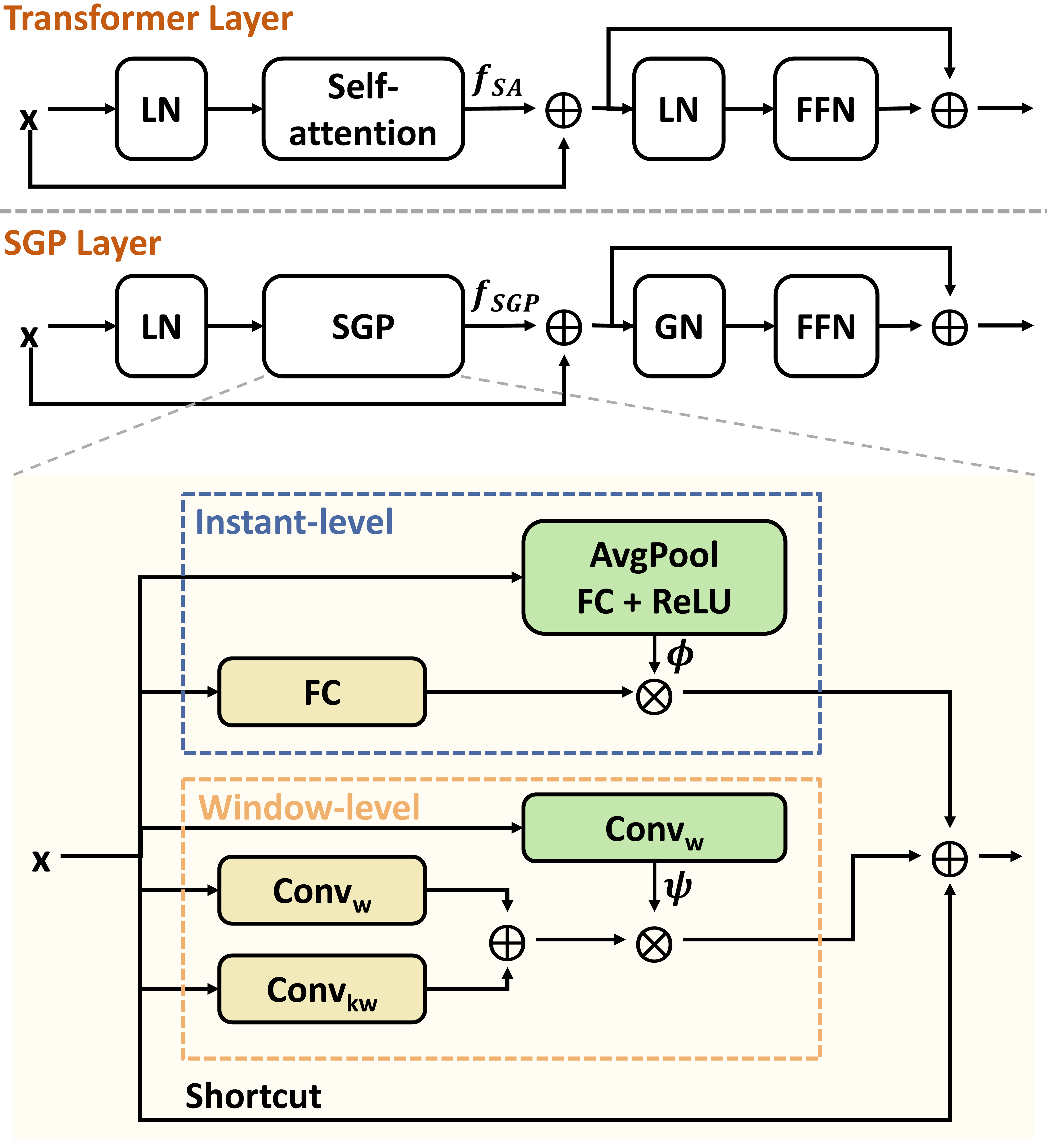}
    
  \caption{Illustration of the structure of SGP layer. We replace the self-attention and the second Layer Normalization (LN) with SGP and Group Normalization (GN), respectively. }
  \label{fig:module}
  }
\end{figure}

\paragraph{The SGP layer.}
Based on the above discovery, we propose a Scalable-Granularity Perception (\modulename) layer to effectively capture the action information and suppress rank loss. The major difference between the Transformer layer and SGP layer is the replacement of the self-attention module with the fully-convolutional module SGP. The successive Layer Normalization\cite{ba2016layer} (LN) is changed to Group Normalization\cite{wu2018group} (GN).

As shown in \figref{fig:module}, \modulename~contains two main branches: an instant-level branch and a window-level branch. 
In the instant-level branch, we aim to increase the feature discriminability between action and non-action instant by enlarging their feature distance with the video-level average feature. The window-level branch is designed to introduce the semantic content from a wider receptive field with a branch $\psi$ to help dynamically focus on the features of which scale.
Mathematically, the \modulename~can be written as:
\begin{equation}
    f_{SGP} = \phi(x)FC(x) + \psi(x)(Conv_w(x) + Conv_{kw}(x))+ x,
\end{equation}
where $FC$ and $Conv_w$ denotes fully-connected layer and the 1-D depth-wise convolution layer~\cite{chollet2017xception} over temporal dimension with window size $w$.
As a signature design of \modulename, $k$ is a scalable factor aiming at capturing a larger granularity of temporal information.
The video-level average feature $\phi(x)$ and branch $\psi(x)$ are given as
\begin{align}
    \phi(x) &= ReLU(FC(AvgPool(x))),\\
    \psi(x) &= Conv_w(x),
\end{align}
where $AvgPool(x)$ is the average pooling for all features over the temporal dimension. Here, both $\phi(x)$ and $\psi(x)$ perform the element-wise multiplication with the mainstream feature.

The resultant \modulename-based feature pyramid can achieve better performance than the transformer-based feature pyramid while being much more efficient.

\subsection{Trident-head with Relative Boundary Modeling}
\label{sec:head}


\paragraph{Intrinsic property of action boundaries.}
Regarding the detection head, some existing methods directly regress the temporal length~\cite{zhang2022actionformer} of the action at each instant of the feature and refine with the boundary feature\cite{lin2021learning,qing2021temporal}, or ~\cite{lin2018bsn,lin2019bmn,zeng2019graph} simply predict an \emph{actionness} score (indicating the probability of being an action). These simple strategies suffer from a problem in practice: imprecise boundary predictions, due to the intrinsic property of actions in videos. Namely, the boundaries of actions are usually not obvious, unlike the boundaries of objects in object detection. Intuitively, a more statistical boundary localization method can reduce uncertainty and facilitate more precise boundaries.


\begin{figure}[t]
    \centering{
    \includegraphics[width=0.90\linewidth]{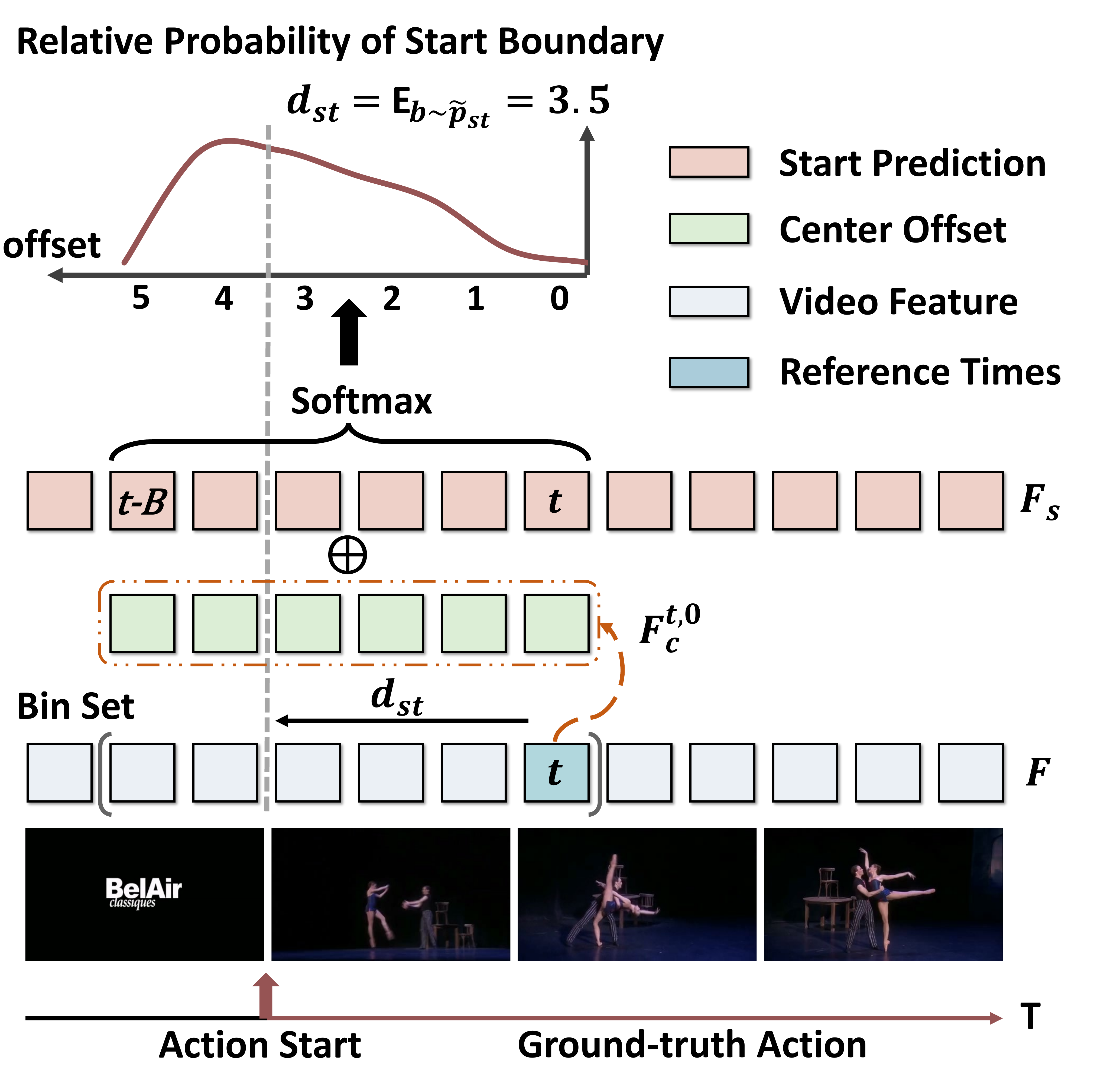}
    
  \caption{The boundary localization mechanism of Trident-head. We predict the boundary response and the center offset for each instant. At the instant t, the predicted boundary response in neighboring bin set is summed element-wise with the center offset corresponding to the instant t, which is further estimated as the relative boundary distribution.
  Finally, the offset is computed based on the expected value of the bin.
  }
  \label{fig:head} 
  }
\end{figure}



\paragraph{Trident-head.}
In this work, we propose a boundary-oriented Trident-head to precisely locate the action boundaries based on the relative boundary modeling, \ie~considering the relation of features in a certain period and obtaining the relative probability of being a boundary for each instant in that period. The Trident-head consists of three components: a start head, an end head, and a center-offset head, which are designed to locate the start boundary, end boundary, and the temporal center of the action, respectively. The Trident-head can be trained end-to-end with the detector. 

Concretely, as shown in \figref{fig:head}, given a sequence of features $F\in\mc{R}^{T \times D}$ output from the feature pyramid, we first obtain three feature sequences from the three branches (namely, $F_{s} \in \mc{R}^{T}$, $F_{e} \in \mc{R}^{T}$ and $F_{c} \in \mc{R}^{T \times 2 \times (B+1)}$), where $B$ is the number of bins for boundary prediction, $F_{s}$ and $F_{e}$ characterize the response value for each instant as the starting or ending point of an action, respectively. In addition, the center-offset head aims at estimating two conditional distributions $P(b_{st}|t)$ and $P(b_{et}|t)$. They represent the probability that each instant (in its set of bins) serves as a boundary when the instant $t$ is the midpoint of an action. Then, we model the boundary distance by combining the outputs of the boundary head and center-offset head: 
\begin{align}
    \widetilde{P}_{st} &= Softmax(F_{s}^{[(t-B):t]} + F_{c}^{t,0}),\\
    d_{st} &= \mb{E}_{b\sim \widetilde{P}_{st}}[b] \approx \sum_{b=0}^B(b\widetilde{P}_{stb}),
\end{align}
where $F_{s}^{[(t-B):t]} \in \mc{R}^{B+1}$, $F_{c}^{t,0} \in \mc{R}^{B+1}$ are the feature of the left adjacent bin set of instant $t$ and the center offsets predicted by instant $t$ only, respectively, and $\widetilde{P}_{st}$ is the \emph{relative probability} which represents the probability of each instant as a start of the action within the bin set. Then, the distance between the instant $t$ and the start instant of action instance $d_{st}$ is given by the expectation of the adjacent bin set. Similarly, the offset distance of the end boundary $d_{et}$ can be obtained by
\begin{align}
    \widetilde{P}_{et} &= Softmax(F_{e}^{[t:(t+B)]} + F_{c}^{t,1}),\\
    d_{et} &= \mb{E}_{b\sim \widetilde{P}_{et}}[b] \approx \sum_{b=0}^B(b\widetilde{P}_{etb\textbf{}})
\end{align}

All heads are simply modeled in three layers convolutional networks and share parameters at all feature pyramid levels to reduce the number of parameters.

\myPara{Combination with feature pyramid.} 
We apply the Trident-head in a pre-defined local bin set, which can be further improved by combining it with the feature pyramid.
In this setting, features at each level of the feature pyramid simply share the same small number of bins $B$ (\eg~16) and then the corresponding prediction for each level $l$ can be scaled by $2^{l-1}$, which can significantly help to stabilize the training process. 

Formally, for an instant in the l-th feature level $t^l$,  \name~estimates the boundary distance $\hat{d}_{st}^l$ and $\hat{d}_{et}^l$ with the Trident-head described above, then the segments $a=(\hat{s}_t, \hat{e}_t)$ can be decoded by 
\begin{align}
    \hat{s}_t &= (t-\hat{d}_{st}^l)\times 2^{l-1},\\
    \hat{e}_t &= (t+\hat{d}_{et}^l)\times 2^{l-1}.
\end{align}

\myPara{Comparison with existing methods that have explicit boundary modeling.} Many previous methods improve boundary predictions. We divide them into two broad categories: the prediction based on sampling instants in segments~\cite{lin2019bmn,liu2022end,shi2022react} and the prediction based on a single instant. The first category predicts the boundary 
according to the global feature of the predicted instance segments.
They only consider global information instead of detailed information at each instant. The second category directly predicts the distance between an instant and its corresponding boundary based on the instant-level feature~\cite{lin2021learning,zhang2022actionformer,zhao2021video,qing2021temporal}. Some of them refine the segment with boundary features~\cite{lin2021learning,qing2021temporal,zhao2021video}. 
However, they do not take the relation (\ie~relative probability of being a boundary) of adjacent instants into account.
The proposed Trident-head differs from these two categories and shows superior performance in precise boundary localization.

\subsection{Training and Inference}
Each layer $l$ of the feature pyramid outputs a temporal feature $F^l \in \mc{R}^{(2^{l-1}T)\times D}$, which is then fed to the classification head and the Trident-head for action instance detection. The output of each instant $t$ in feature pyramid layer $l$ is denoted as $\hat{o}_{t}^l = (\hat{c}_{t}^l, \hat{d}_{st}^l, \hat{d}_{et}^l)$. 

The overall loss function is then defined as follows:
\begin{align}
\begin{aligned}
\mc{L}&=\frac{1}{N_{pos}}\sum_{l,t}\mathbbm{1}_{\{c^l_t>0\}}(\sigma_{IoU}\mc{L}_{cls} + \mc{L}_{reg})\\
&+ \frac{1}{N_{neg}}\sum_{l,t}\mathbbm{1}_{\{c^l_t=0\}}\mc{L}_{cls},
\end{aligned}
\end{align}
where $\sigma_{IoU}$ is the temporal IoU between the predicted segment and the ground truth action instance, and $\mc{L}_{cls}$, $\mc{L}_{reg}$ is focal loss~\cite{lin2017focal} and IoU loss~\cite{rezatofighi2019generalized}. $N_{pos}$ and $N_{neg}$ denote the number of positive and negative samples.
The term $\sigma_{IoU}$ is used to reweight the classification loss at each instant, such that instants with better regression (\ie~of higher quality) contribute more to the training. 
Following previous methods~\cite{tian2019fcos,zhang2020bridging,zhang2022actionformer}, center sampling is adopted to determine the positive samples. Namely, the instants around the center of an action instance are labeled as positive and all the others are considered as negative.

\myPara{Inference.} At inference time, the instants with classification scores higher than threshold $\lambda$ and their corresponding instances are kept. Lastly, Soft-NMS~\cite{bodla2017soft} is applied for the deduplication of predicted instances.

\section{Experiments}
\label{sec:exp}
\myPara{Datasets.}
We conduct experiments on four challenging datasets: THUMOS14~\cite{THUMOS14}, ActivityNet-1.3\cite{caba2015activitynet}, HACS-Segment~\cite{zhao2019hacs} and EPIC-KITCHEN 100~\cite{Damen2022RESCALING}. THUMOS14 consists of 20 sport action classes and it contains 200 and 213 untrimmed videos with 3,007 and 3,358 action instances on the training set and test set, respectively. ActivityNet and HACS are two large-scale datasets and they share 200 classes of action. They have 10,024 and 37,613 videos for training, as well as 4,926 and 5,981 videos for test. The EPIC-KITCHEN 100 is a large-scale dataset in first-person vision, which have two sub-tasks: \emph{noun} localization (\eg~door) and \emph{verb} localization (\eg~open the door). It contains 495 and 138 videos with 67,217 and 9,668 action instances for training and test, respectively. The number of action classes for \emph{noun} and \emph{verb} are 300 and 97.  

\myPara{Evaluation.} For all these datasets, only the annotations of the training and validation sets are accessible. Following the previous practice~\cite{lin2019bmn, zhang2022actionformer,cheng2022tallformer,zeng2019graph}, we evaluate on the validation set. We report the mean average precision (mAP) at different intersection over union (IoU) thresholds. For THUMOS14 and EPIC-KITCHEN, we report the IoU thresholds at [0.3:0.7:0.1] and [0.1:0.5:0.1] respectively. For ActivityNet and HACS, we report the result at IoU threshold [0.5, 0.75, 0.95] and the avearge mAP is computed at [0.5:0.95:0.05]. 

\begin{table}[t]
\centering{
\caption{Comparison with the state-of-the-art methods on THUMOS14 dataset. *: TSN backbone. $\dagger$: Swin Transformer backbone. Others: I3D backbone.}
\label{table:thumos14}
\setlength{\tabcolsep}{1.3mm}
\scalebox{0.9}{
\begin{tabular}{c|c c c c c c }
\toprule
Method & 0.3 & 0.4 & 0.5 & 0.6 & 0.7 & Avg.\\
\midrule
 BMN~\cite{lin2019bmn}* & 56.0 & 47.4 & 38.8 & 29.7 & 20.5 & 38.5 \\
 G-TAD~\cite{xu2020g}*  & 54.5 & 47.6 & 40.3 & 30.8 & 23.4 & 39.3 \\
 A2Net~\cite{yang2020revisiting} & 58.6 & 54.1 & 45.5 & 32.5 & 17.2 & 41.6 \\
 TCANet~\cite{qing2021temporal}* & 60.6 & 53.2 & 44.6 & 36.8 & 26.7 & 44.3 \\
 RTD-Net~\cite{tan2021relaxed} & 68.3 & 62.3 & 51.9 & 38.8 & 23.7 & 49.0 \\
 VSGN~\cite{zhao2021video}* & 66.7 & 60.4 & 52.4 & 41.0 & 30.4 & 50.2 \\
 ContextLoc~\cite{zhu2021enriching} & 68.3 & 63.8 & 54.3 & 41.8 & 26.2 & 50.9\\
 AFSD~\cite{lin2021learning} & 67.3 & 62.4 & 55.5 & 43.7 & 31.1 & 52.0 \\
 ReAct~\cite{shi2022react}* & 69.2 & 65.0 & 57.1 & 47.8 & 35.6 & 55.0 \\
 TadTR~\cite{liu2022end} & 74.8 & 69.1 & 60.1 & 46.6 & 32.8 & 56.7\\
TALLFormer~\cite{cheng2022tallformer}$\dagger$ & 76.0 & - & 63.2 & - & 34.5 & 59.2\\
 ActionFormer~\cite{zhang2022actionformer} & 82.1 & 77.8 & 71.0 & 59.4 & 43.9 & 66.8\\
\textbf{TriDet} & \textbf{83.6} & \textbf{80.1} & \textbf{72.9} & \textbf{62.4} & \textbf{47.4} & \textbf{69.3}\\
\bottomrule
\end{tabular}
}}
\end{table}

\begin{table}[t]
\centering{
\caption{Comparison with the state-of-the-art methods on HACS dataset.}
\label{table:hacs}
\setlength{\tabcolsep}{1.3mm}
\renewcommand{\arraystretch}{1.1}
\scalebox{0.9}{
\begin{tabular}{c|c |c  c c c c }
\toprule
Method & Backbone & 0.5 & 0.75 & 0.95 & Avg. \\
\midrule
SSN~\cite{SSN2017ICCV} & I3D & 28.8 & 18.8 & 5.3 &  19.0\\
LoFi~\cite{xu2021low} & TSM & 37.8 & 24.4 & 7.3 & 24.6\\
G-TAD~\cite{xu2020g} & I3D & 41.1 & 27.6 & 8.3 &  27.5\\
TadTR~\cite{liu2022end} & I3D & 47.1 & 32.1 & 10.9 & 32.1\\
BMN~\cite{lin2019bmn} & SlowFast &  52.5 & 36.4 & 10.4 & 35.8\\
TALLFormer~\cite{cheng2022tallformer} & Swin &  55.0 & 36.1 & 11.8 & 36.5\\
TCANet~\cite{qing2021temporal} & SlowFast &  54.1 & 37.2 & 11.3 & 36.8\\
\textbf{TriDet} & I3D & \textbf{54.5} & \textbf{36.8} & \textbf{11.5} & \textbf{36.8}\\
\textbf{TriDet} & SlowFast & \textbf{56.7} & \textbf{39.3} & \textbf{11.7} & \textbf{38.6}\\
\bottomrule
\end{tabular}}
}
\end{table}

\subsection{Implementation Details}
TriDet is trained end-to-end with AdamW~\cite{loshchilov2018decoupled}
optimizer. The initial learning rate is set to $10^{-4}$ for THUMOS14 and EPIC-KITCHEN, and  $10^{-3}$ for ActivityNet and HACS. We detach the gradient before the start boundary head and end boundary head and initialize the CNN weights of these two heads with a Gaussian distribution $\mc{N}(0, 0.1)$ to stabilize the training process. The learning rate is updated with Cosine Annealing schedule~\cite{loshchilov2017sgdr}. We train 40, 23, 19, 15 and 13 epochs for THUMOS14, EPIC-KITCHEN \textbf{verb}, EPIC-KITCHEN \emph{noun}, ActivityNet and HACS (containing warmup 20, 5, 5, 10, 10 epochs).

For ActivityNet and HACS, the number of bins $B$ of the Trident-head is set to 12, 14 and the convoluntion window $w$ is set to 15, 11 and the scale factor $k$ is set to 1.3 and 1.0, respectively. 
For THUMOS14 and EPIC-KITCHEN, the number of bins $B$ of the Trident-head is set to 16 and the convoluntion window $w$ is set to 1 and the scale factor $k$ is set to 1.5. We round the scaled windows size and take it up to the nearest odd number for convenience. We conduct our experiments on a single NVIDIA A100 GPU.

\begin{table}[t]
\centering{
\caption{Comparison with the state-of-the-art methods on EPIC-KITCHEN dataset. \emph{V.} and \emph{N.} denote the \emph{verb} and \emph{noun} sub-tasks, respectively.}
\label{table:epic}
\setlength{\tabcolsep}{0.9mm}
\renewcommand{\arraystretch}{1.1}
\scalebox{0.9}{
\begin{tabular}{c|c|c c c c c c}
\toprule
 &Method & 0.1 & 0.2 & 0.3 & 0.4 & 0.5 &Avg. \\
\midrule
\multirow{4}*{\tabincell{c}{\emph{V.}}}
&BMN~\cite{lin2019bmn}& 10.8 & 8.8 & 8.4 & 7.1& 5.6 & 8.4 \\
&G-TAD~\cite{xu2020g}& 12.1 & 11.0 & 9.4 & 8.1 & 6.5 & 9.4 \\
&ActionFormer~\cite{zhang2022actionformer}& 26.6 & 25.4 & 24.2 & 22.3 & 19.1 & 23.5 \\
&\textbf{TriDet}&\textbf{ 28.6}  & \textbf{27.4}  & \textbf{26.1} & \textbf{24.2}  & \textbf{20.8}  & \textbf{25.4} \\
\midrule
\multirow{4}*{\tabincell{c}{\emph{N.}}}
&BMN ~\cite{lin2019bmn}& 10.3 & 8.3 & 6.2 & 4.5 & 3.4 & 6.5 \\
&G-TAD~\cite{xu2020g}& 11.0 & 10.0 & 8.6 & 7.0 & 5.4 & 8.4 \\
&ActionFormer~\cite{zhang2022actionformer}& 25.2 & 24.1 & 22.7 & 20.5 & 17.0 & 21.9 \\
&\textbf{TriDet}& \textbf{27.4}  & \textbf{26.3}  & \textbf{24.6} & \textbf{22.2}  & \textbf{18.3}  & \textbf{23.8} \\
\bottomrule
\end{tabular}
}}
\end{table}
\subsection{Main Results}

\myPara{THUMOS14.} We adopt the commonly used I3D\cite{carreira2017quo} as our backbone feature and \tabref{table:thumos14} presents the results. Our method achieves an average mAP of $69.3\%$, outperforming all previous methods including one-stage and two-stage methods. Notably, our method also achieves better performance than recent Transformer-based methods~\cite{zhang2022actionformer,cheng2022tallformer,shi2022react,liu2022end,qing2021temporal}, which demonstrates that the simple design can also have impressive results. 

\myPara{HACS.} For the HACS-segment dataset, we conduct experiments based on two commonly used features: the official I3D\cite{carreira2017quo} feature and the SlowFast~\cite{feichtenhofer2019slowfast} feature. As shown in \tabref{table:hacs}, our method achieves an average mAP of $36.8\%$ with the official features. It is the state-of-the-art and outperforms the previous best model TadTR by about $4.7\%$ in average mAP. We also show that changing the backbone to SlowFast can further boost performance, resulting in a $1.8\%$ increase in average mAP, which indicates that our method can benefit from a much more advanced backbone network.

\myPara{EPIC-KITCHEN.} On this dataset, following all previous methods, SlowFast is adopted as the backbone feature. The method of our main comparison is ActionFormer\cite{zhang2022actionformer}, which has demonstrated promising performance in EPIC-KITCHEN 100 dataset. We present the results in~\tabref{table:epic}. Our method shows a significant improvement in both sub-tasks: \emph{verb} and \emph{noun}, and achieves $25.4\%$ and $23.8\%$ average mAP, respectively. Note that our method outperforms ActionFormer with the same features by a large margin ($1.9\%$ and $1.9\%$ average mAP in \emph{verb} and \emph{noun}, respectively). Moreover, our method achieves state-of-the-art performance on this challenging dataset. 


\begin{table}[t]
\centering{
\caption{Comparison with the state-of-the-art methods on ActivityNet-1.3 dataset.}
\label{table:activitynet}
\setlength{\tabcolsep}{1.3mm}
\renewcommand{\arraystretch}{1.1}
\scalebox{0.9}{
\begin{tabular}{c| c | c c c c c }
\toprule
Method & Backbone& 0.5 & 0.75 & 0.95 & Avg. \\
\midrule
PGCN~\cite{zeng2019graph} & I3D & 48.3 & 33.2 & 3.3 & 31.1\\
ReAct~\cite{shi2022react} & TSN & 49.6 & 33.0 & 8.6 & 32.6\\
BMN~\cite{lin2019bmn}& TSN & 50.1 & 34.8 & 8.3 & 33.9 \\
G-TAD~\cite{xu2020g}& TSN & 50.4 & 34.6 & 9.0 & 34.1\\
AFSD~\cite{lin2021learning} & I3D & 52.4 & 35.2 & 6.5 & 34.3\\
TadTR~\cite{liu2022end} & TSN & 51.3  & 35.0  & 9.5  & 34.6\\
TadTR~\cite{liu2022end} & R(2+1)D & 53.6  & 37.5  & 10.5  & 36.8\\
VSGN~\cite{zhao2021video} & I3D & 52.3 & 35.2 & 8.3 & 34.7\\
PBRNet~\cite{liu2020progressive} & I3D & 54.0 & 35.0 & 9.0 & 35.0\\
TCANet+BMN~\cite{qing2021temporal} & TSN & 52.3 & 36.7 & 6.9 & 35.5\\
TCANet+BMN~\cite{qing2021temporal} & SlowFast & 54.3 & \textbf{39.1} & \textbf{8.4} & \textbf{37.6}\\
TALLFormer~\cite{cheng2022tallformer} & Swin & 54.1 & 36.2 & 7.9 & 35.6\\
ActionFormer~\cite{zhang2022actionformer} & R(2+1)D & \textbf{54.7} & 37.8 & \textbf{8.4} & 36.6\\
\textbf{TriDet} & R(2+1)D &\textbf{54.7} & \textbf{38.0} & \textbf{8.4} & \textbf{36.8}\\
\bottomrule
\end{tabular}
}}
\end{table}

\myPara{ActivityNet.} For the ActivityNet v1.3 dataset, we adopt the TSP R(2+1)D~\cite{alwassel2021tsp} as our backbone feature. Following previous methods\cite{zhang2022actionformer,cheng2022tallformer, qing2021temporal, liu2022end,lin2021learning}, the video classification score predicted from the UntrimmedNet is adopted to multiply with the final detection score. \tabref{table:activitynet} presents the results. Our method still shows a promising result: \name~outperform the second best model~\cite{zhang2022actionformer} with the same feature, only worse than TCANet~\cite{qing2021temporal} which is a two-stage method and using the SlowFast as the backbone feature which is not available now.

\subsection{Ablation Study}
In this section, we mainly conduct the ablation studies on the THUMOS14 dataset.

\myPara{Main components analysis.}
We demonstrate the effectiveness of our proposed components in \name: SGP layer and Trident-head. To verify the effectiveness of our SGP layer, we use a baseline feature pyramid used by~\cite{lin2021learning, zhang2022actionformer} to replace our SGP layer. The baseline consists of two 1D-convolutional layers and shortcut. The window size of convolutional layers is set to 3 and the number of channels of the intermediate feature is set to the same dimension as the intermediate dimension in the FFN in our SGP layer. All other hyperparameters (\eg~number of the pyramid layers, etc.) are set to the same as our framework.

As depicted in \tabref{table:ablation}, compared with the baseline model we implement (Row 1), the SGP layer brings a $6.2\%$ absolute improvement in the average mAP. Secondly, we compare the SGP with the previous state-of-the-art method, ActionFormer, which adopts a self-attention mechanism in a sliding window behavior\cite{beltagy2020longformer} with window size $7$ (Row 2). We can see our SGP layer still has $1.5\%$ improvement in average mAP, demonstrating that the convolutional network can also have excellent performance in TAD task. Besides, we compare our Trident-head with the normal instant-level regression head, which regresses the boundary distance for each instant. We can see that the Trident-head improves the average mAP by $1.0\%$, and the mAP improvement is more obvious in the case of high IoU threshold (\eg~$1.6\%$ average mAP improvement in IoU 0.7).

\begin{table}[t]
\centering{
\caption{Analysis of the Effectiveness of three main components on THUMOS14.}
\label{table:ablation}
\setlength{\tabcolsep}{1.0mm}
\renewcommand{\arraystretch}{1.2}
\scalebox{0.9}{
\begin{tabular}{c| c  c  c| c c c c}
\toprule
Method & SA & SGP & Trident & 0.3 & 0.5 & 0.7 & Avg. \\
\midrule
1 & &  &  &  77.3 & 65.2 & 40.0 & 62.1 \\
2 & $\surd$  &  & & 82.1 & 71.0 & 43.9 & 66.8 \\
3 & & $\surd$ & & \textbf{83.6} & 71.7 & 45.8 & 68.3 \\
4 & & $\surd$ & $\surd$  &\textbf{ 83.6}  & \textbf{72.9} & \textbf{47.4} & \textbf{69.3} \\
\bottomrule
\end{tabular}
}}
\end{table}

\begin{table}[t]
\centering{
\caption{Analysis of computation cost on THUMOS14. Main: All parts of the model except the detection head. *: Our method with a normal instant-level regression head.}
\label{table:time}
\setlength{\tabcolsep}{0.9mm}
\renewcommand{\arraystretch}{1.3}
\scalebox{0.9}{
\begin{tabular}{c|c c c | c c c | c}
\toprule
\multirow{2}*{\tabincell{c}{Method}} & \multicolumn{3}{c|}{mAP} & \multicolumn{3}{c|}{GMACs} & \multirow{2}*{\tabincell{c}{Latency\\
(ms)}}\\
 & 0.3 & 0.7 & Avg. & Main & Head & All\\
\midrule
ActionFormer & 82.1 & 43.9 & 66.8 & 30.8 & \textbf{14.4} & 45.3 & 224\\
\textbf{TriDet$^*$} & \textbf{83.6} & 45.8 & 68.3 & \textbf{14.5} & \textbf{14.4} & \textbf{28.9} & \textbf{145}\\
\textbf{TriDet} & \textbf{83.6} & \textbf{47.4} & \textbf{69.3} & \textbf{14.5} & 29.1 & 43.7 & 167\\
\bottomrule
\end{tabular}
}}
\end{table}


\begin{figure}[t]
    \centering

    \subcaptionbox{The choice of w (k=1)\label{fig:w}}{
        \includegraphics[width=0.8\linewidth]{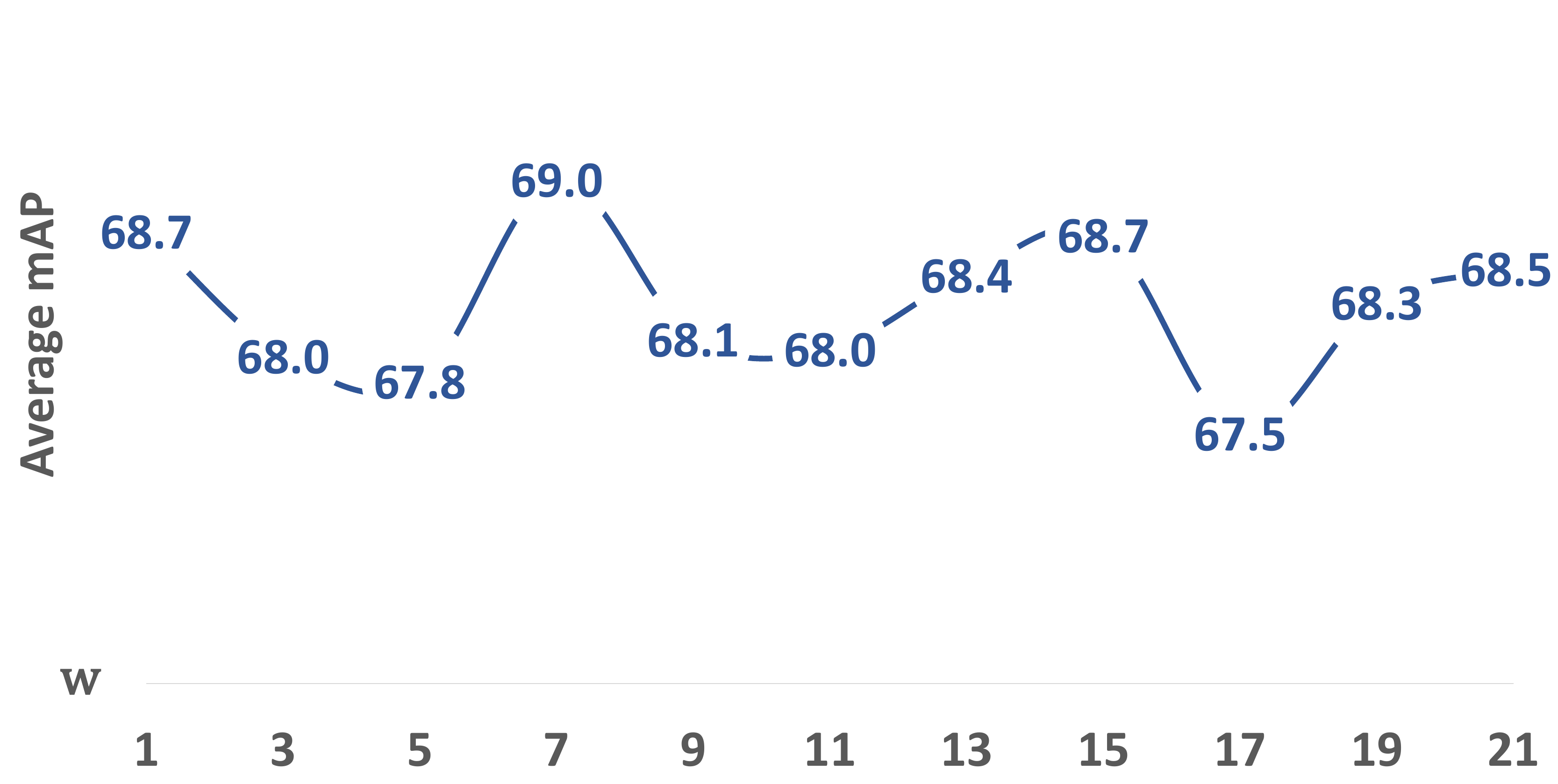}
    }
    
    \subcaptionbox{The choice of k (w=1)\label{fig:k}}{
        \includegraphics[width=0.8\linewidth]{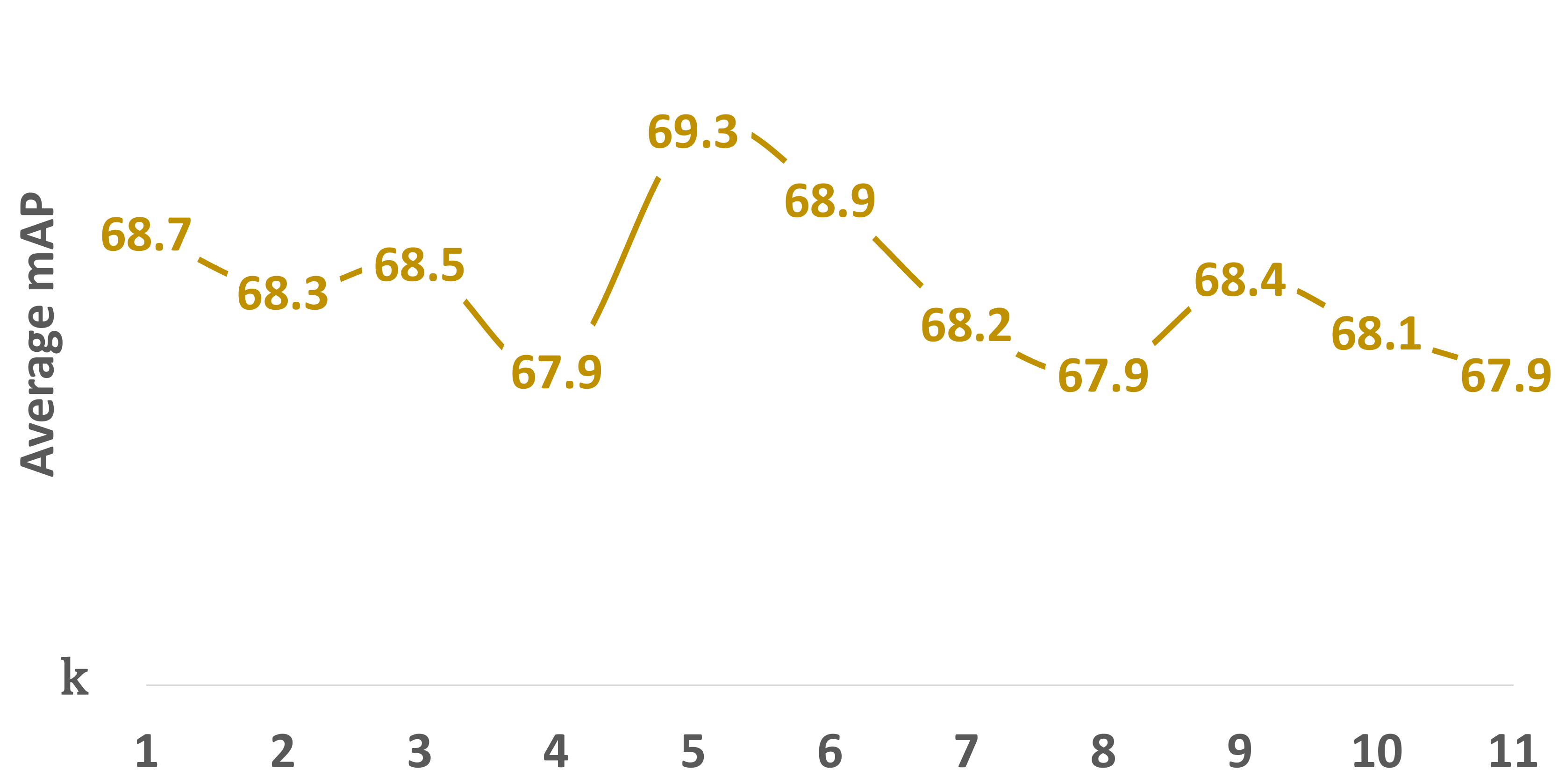}
    }

  \caption{Effectiveness of window size w and k.}
  \label{fig:statics}
\end{figure}

\myPara{Computational complexity.}
We compare the computational complexity and latency of \name~with the recent ActionFormer\cite{zhang2022actionformer}, which brings a large improvement to TAD by introducing the Transformer-based feature pyramid.

As shown in \tabref{table:time}, we divide the detector into two parts: the main architecture and the detection heads (\eg~classification head and regression head). We report the GMACs for each part and the inference latency (average over five times) on THUMOS14 dataset using an input with the shape $2304 \times 2048$, following the \cite{zhang2022actionformer}. We also report our results using the Trident-head and the normal regression head, respectively. First, from the first row, we see that GMACs of our main architecture with SGP layer is only $47.1\%$ of the ActionFormer (14.5 versus 30.8), and the overall latency is only $65.2\%$ (146ms versus 224ms), but \name~still outperforms Actionformer by $1.5\%$ average mAP, which shows that our main architecture is much better than the local Transformer-based method. Besides, we further evaluate our method with Trident-head. The experimental result shows that our framework can be improved by the Trident-head which further brings $1.0\%$ average mAP improvement and the GMACs is still $1.6$G smaller than ActionFormer, and the latency is still only $74.6\%$ of it, proving the high efficiency of our method. 

\begin{table}[t]
\centering{
\caption{Analysis of the number of feature pyramid layers.}
\label{table:fpn}
\setlength{\tabcolsep}{1.3mm}
\renewcommand{\arraystretch}{1.1}
\scalebox{0.9}{
\begin{tabular}{c |c | c  c c |c |c c c}
\toprule
$\#$Levels & Bin & 0.3  & 0.7 & Avg. & Bin & 0.3  & 0.7 & Avg. \\
\midrule
1 & \multirow{7}*{\tabincell{c}{16}} & 70.1 & 15.3 & 44.5 & 512 & 74.2 & 25.9 & 53.5 \\
2 & & 77.9 & 27.8 & 57.1 & 256 & 78.0 & 29.7 & 58.0\\
3 & & 79.8 & 37.7 & 61.8 & 128 & 80.6 & 37.1 & 62.5\\
4 & & 82.1 & 42.6 & 66.1 & 64 & 82.7 & 39.0 & 64.7\\
5 & & 82.9 & 45.7 & 68.1 & 32 & 82.7 & 44.7 & 67.4\\
6 & & \textbf{83.6}  & \textbf{47.4} & \textbf{69.3} & 16
& \textbf{83.6} & \textbf{47.4} & \textbf{69.3}\\
7 & & 83.4 & 46.2 & 68.9 & 8 & 82.7 & 46.8 & 68.2 \\


\bottomrule
\end{tabular}
}}
\end{table}


\myPara{Ablation on the window size in SGP layer.}
In this section, we study the effectiveness of the two hyper-parameters related to the window size in the SGP layer. Firstly, we fix $k=1$ and vary $w$. Secondly, we fix the value of $w=1$ and change $k$. Finally, we present the results in \figref{fig:statics} on THUMOS14 datasets. We find that different choices of $w$ and $k$ produce stable results on both datasets. The optimal values are $w=1$, $k=5$ for THUMOS14.

\myPara{The effectiveness of feature pyramid level.}
To study the effectiveness of the feature pyramid and its relation with the number of Trident-head bin set, we start the ablation from the feature pyramid with 16 bins and 6 levels. We conduct two sets of experiments: a fixed number of bins or a scaled number of bins for each level in the feature pyramid. 
As shown in \tabref{table:fpn}, we can see that the detection performance rises as the number of layers increases. With fewer levels (\ie~level less than 3), more bins bring better performance. That is because the fewer the number of levels, the more bins are needed to predict the action with a long duration (\ie~higher resolution at the highest level). We achieve the best result with a level number of 6.  

\begin{table}[t]
\centering{
\caption{Analysis of the number of bins.}
\label{table:bin}
\setlength{\tabcolsep}{1.3mm}
\renewcommand{\arraystretch}{1.3}
\scalebox{0.9}{
\begin{tabular}{c |c c c c | c  c c c }
\toprule
\multirow{2}*{\tabincell{c}{Bin}} & \multicolumn{4}{c|}{THUMOS14} & \multicolumn{4}{c}{HACS} \\
 & 0.3 & 0.5 & 0.7 & Avg. & 0.5 & 0.75 & 0.95 & Avg.\\
\midrule
4  & 82.9 & 71.5 & 46.3 & 68.1 & 55.7 & 32.3 & 4.7 & 33.3 \\
8  & 83.5 & \textbf{72.9} & 46.3 & 69.0 & 56.2 & 38.4 & 11.2 & 38.0 \\
10 & 82.8 & 71.8 & 46.2 & 68.1 & 56.2 & 38.5 & 11.1 & 37.9 \\
12 & \textbf{83.6} & 72.3 & 46.2 & 68.5 & 56.3 & 38.4 & 11.1 & 38.0 \\
14 & 83.4 & 72.6 & 45.6 & 68.3 & \textbf{56.7} & \textbf{39.3} & \textbf{11.7} &  \textbf{38.6} \\
16 & \textbf{83.6} & \textbf{72.9} & \textbf{47.4} & \textbf{69.3} & 56.5 & 38.6 & 11.1 & 38.1 \\
20 & \textbf{83.6} & 71.7 & 45.8 & 68.3 & 56.3 & 38.6 & 11.1 & 38.0 \\
\bottomrule
\end{tabular}}
}
\end{table}

\myPara{Ablation on the number of bins.}
In this section, we present the ablation results for the choice of the number of bins on the THUMOS14 and HACS datasets in \tabref{table:bin}. We observe the optimal value is obtained at 16 and 14 on the THUMOS14 and the HACS, respectively. We also find that a small bin value leads to significant performance degradation on HACS but not on THUMOS14. That is because the THUMOS14 dataset aims at detecting a large number of action segments from a long video and a small bin value can meet the requirements, but on HACS, there are more actions with long duration, thus a larger number of bins is needed. 




\section{Conclusion}
In this paper, we aim at improving the temporal action detection task with a simple one-stage convolutional-based framework TriDet with relative boundary modeling. Experiments conducted on THUMOS14, HACS, EPIC-KITCHEN and ActivityNet demonstrate a high generalization capability of our method, which achieves state-of-the-art performance on the first three datasets and comparable results on ActivityNet. Extensive ablation studies are conducted to verify the effectiveness of each proposed component.

\noindent\textbf{Acknowledgement.} 
This work is supported by the National Natural Science Foundation of China under Grant 62132002.

{\small
\bibliographystyle{ieee_fullname}
\bibliography{egbib}
}
\clearpage
\appendix

\section{Supplementary Material}
\subsection{Network Architecture in Feature Pyramid}

\label{sec:self2cnn}
\paragraph{From Transformer to CNN.}
To be self-contained, we analyze the impact of module design on the detector. 
For comparison, we build two baseline models: a convolutional baseline and a Transformer baseline. Firstly, we build the convolution baseline where the convolutional module is adopted from the previous one-stage detector~\cite{lin2021learning,zhang2022actionformer}. Secondly, the previous state-of-the-art detector~\cite{zhang2022actionformer} with the local window self-attention~\cite{beltagy2020longformer} is chosen as the Transformer baseline. Then, to analyze the importance of two common components: self-attention and normalization, in the Transformer~\cite{vaswani2017attention} macrostructure, we provide three variants of the convolutional-based structure: SA-to-CNN, LN-to-GN and LN-GN-Mix, as~\figref{fig:self2cnn} shown, and validate their performance on THUMOS14. 

\paragraph{Results.}
From the~\tabref{table:self2cnn}, we can see there is a large performance gap between the Transformer baseline and the CNN baseline (about $8.1\%$ in average mAP), demonstrating that the Transformer holds a large advantage for TAD tasks. 
Then, we conduct the ablation study with the three variants with normal regression head and Trident-head, respectively.

We first simply replace the local self-attention with a 1D convolutional layer which has the same receptive field with~\cite{zhang2022actionformer} (\eg~kernel size is 19). This change brings a dramatic performance increase in average mAP compared with the CNN baseline (about $6.2\%$) but is still behind the Transformer baseline by about $1.9\%$. Next, we conduct experiments with different normalization layers (\ie~Layer Normalization (LN)~\cite{ba2016layer} and Group Normalization (GN)~\cite{wu2018group}) and we find that the hybrid structure of LN and GN (LN-GN-Mix) shows better performance comparing to the original form of the Transformer ($65.7$ versus $64.9$). 
By combining with the Trident-head, the LN-GN-Mix version achieves $66.0\%$ in average mAP, which demonstrates the possibility of efficient convolutional modeling. These empirical results further motivate us to improve the feature pyramid with SGP layer (see Sec 3.2 of the main test for more details).

\begin{figure}[t!]
    \centering{
    \includegraphics[width=0.9\linewidth]{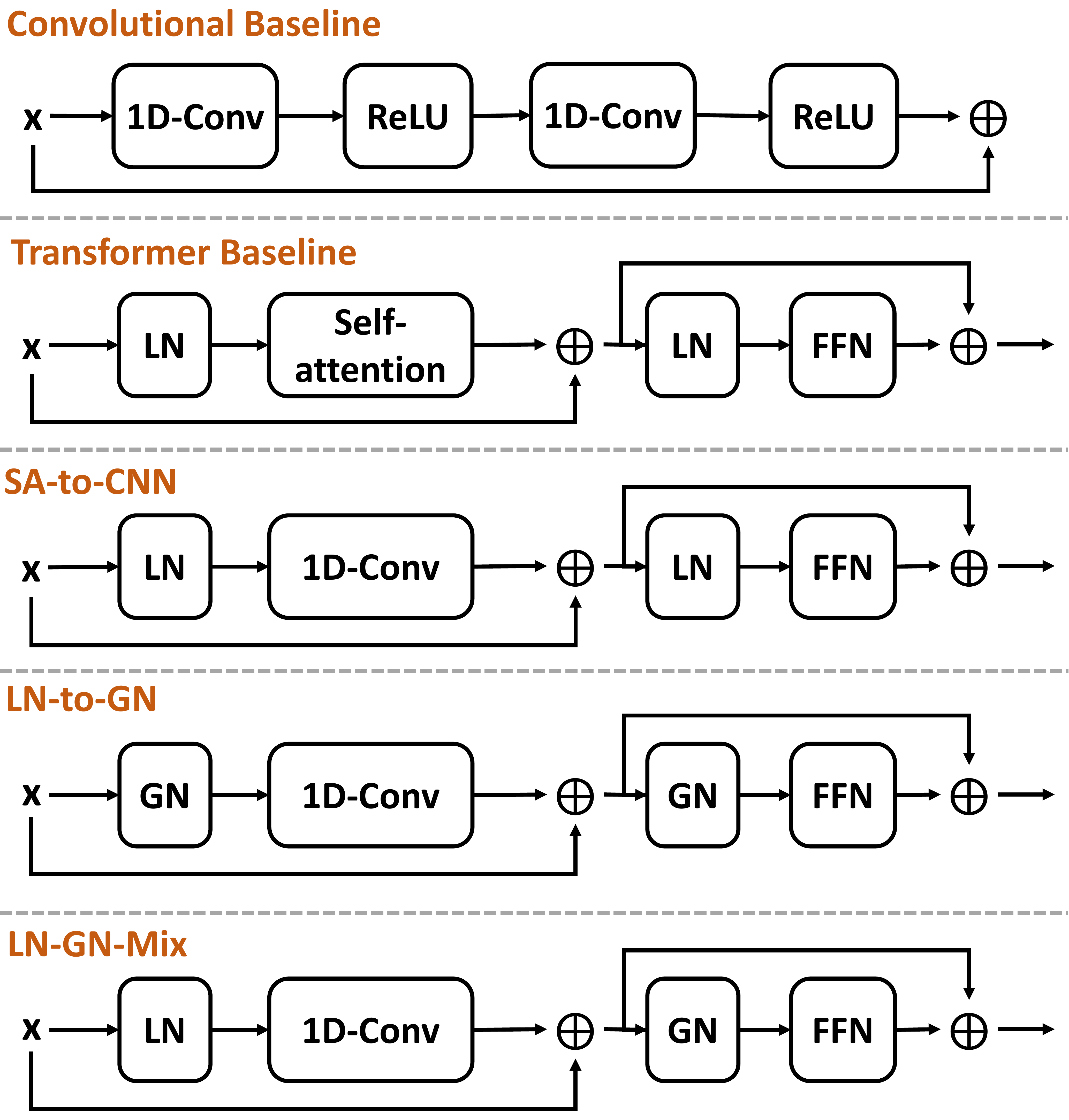}
    
  \caption{Two baseline models and three different variants of the convolutional-based structure. }
  \label{fig:self2cnn} 
  }
\end{figure}

\begin{table}[t]
\centering{
\caption{The results of different variants on THUMOS14. *: with Trident-head.}
\label{table:self2cnn}
\setlength{\tabcolsep}{1.3mm}
\begin{tabular}{c|c c c c c c }
\toprule
Method & 0.3 & 0.4 & 0.5 & 0.6 & 0.7 & Avg.\\
\midrule
 CNN Baseline & 73.7 & 68.8 & 61.4 & 51.6 & 38.0 & 58.7 \\
 Transformer Baseline & 82.1 & 77.8 & 71.0 & 59.4 & 43.9 & 66.8 \\
 \midrule
 SA-to-CNN & 80.4 & 76.4 & 67.5 & 57.5 & 42.9 & 64.9 \\
 LN-to-GN & 80.0 & 76.3 & 68.0 & 57.2 & 42.3 & 64.8 \\
 LN-GN-Mix & 80.8 & 77.2 & 68.8 & 58.1 & 43.6 & 65.7 \\
 \midrule
  SA-to-CNN* & 81.2 & 77.3 & 68.7 & 58.0 & 43.5 & 65.7 \\
 LN-to-GN* & 80.7 & 76.9 & 69.1 & 58.0 & 42.2 & 65.4 \\
 LN-GN-Mix* & 81.6 & 77.7 & 69.5 & 58.2 & 42.9 & 66.0 \\
\bottomrule
\end{tabular}
}
\end{table}

\subsection{The rank loss problem in Transformer.}
In~\cite{dong2021attention}, the authors discuss how the pure self-attention operation causes the input feature to converge to a rank-1 matrix at a double exponential rate, while MLP and residual connections can only partially slow down this convergence. This phenomenon is disastrous for TAD tasks, as the video feature sequences extracted by pre-trained action recognition networks are often highly similar (see Section 1), which further aggravates the rank loss problem and makes the features at each instant indistinguishable, resulting in inaccurate detection of action.

We posit that the core reason for this issue lies in the softmax function used in self-attention. Namely, the probability matrix (\ie~softmax($QK^T$)) is \emph{non-negative} and \emph{the sum of each row is 1}, indicating the outputs of SA are \emph{convex combination} for the value feature $V$. We will demonstrate that the largest angle between any two features in $V' = SA(V)$ is always less than or equal to the largest angle between features in $V$.

\begin{definition}[Convex Combination]
Given a set of points $S=\{x_1, x_2..., x_n\}$, a convex combination is a point of the form $\sum_{n}{a_nx_n}$, where $a_n\geq0$ and $\sum_{n}{a_n}=1$.
\end{definition}

\begin{definition}[Convex Hull]
The convex hull $H$ of a given set of points $S$ is identical to the set of all their convex combinations. A Convex hull is a convex set.
\end{definition}

\begin{property}[Extreme point]
An extreme point $p$ is a point in the set that does not lie on any open line segment between any other two points of the same set. For a point set $S$ and its convex hull $H$, we have $p\in S$.
\end{property}

\begin{lemma} 
\label{lemma:two}
Consider the case of a convex hull that does not contain the origin.
Let $a, b \in \mathbb{R}^n$ and let $S$ be the convex hull formed by them. Then, the angle between any two position vectors of points in $S$ is less than or equal to the angle between the position vectors of the extreme points $\vec{a}$ and $\vec{b}$.
\end{lemma}

\begin{proof}
Consider the objective function
\begin{align*}
    f(x) = \cos{(\vec{x},\vec{y})} = \frac{\langle\vec{x},\vec{y}\rangle}{\left\lVert\vec{x}\right\rVert_2\left\lVert\vec{y}\right\rVert_2},
\end{align*}
where $\vec{x},\vec{y}$ are the position vectors of two points $x_1, x_2$ within the convex hull $S$ (a line segment with extreme points $a$ and $b$). The angle between two vectors is invariant with respect to the magnitude of the vectors, thus, for simplicity, we define $\vec{x}=\vec{a}+x\vec{b}$, $\vec{y}=\vec{a}+y\vec{b}$, where $x,y \in[0,+\infty)$.
Moreover, we have 
\begin{align*}
    f'(x) =&\left\lVert\vec{x}\right\rVert^{-3}_2 \left\lVert\vec{y}\right\rVert^{-1}_2 \times\\
    &[{\langle\Vec{b},\vec{y}\rangle||\vec{a}+x\vec{b}||_2^2-(||\vec{b}||^{2}_2x+\langle\vec{a},\vec{b}\rangle)\langle\vec{a}+x\vec{b},\vec{y}\rangle}]
\end{align*}
We consider 
\begin{align*}
    g(x)=&\langle\Vec{b},\vec{y}\rangle||\vec{a}+x\vec{b}||_2^2-(||\vec{b}||^{2}_2x+\langle\vec{a},\vec{b}\rangle)\langle\vec{a}+x\vec{b},\vec{y}\rangle \\
    =&\langle\vec{b},\vec{y}\rangle(||\vec{a}||_2^2+2\langle\vec{a},\vec{b}\rangle x+||\vec{b}||_2^2x^2)-[\langle\vec{b},\vec{y}\rangle||b||_2^2x^2\\
    &+(\langle\vec{a},\vec{b}\rangle||b||_2^2+\langle\vec{a},\vec{b}\rangle\langle\vec{b},\vec{y}\rangle)x+\langle\vec{a},\vec{y}\rangle\langle\vec{a},\vec{b}\rangle]\\
    =&(\langle\vec{a},\vec{b}\rangle\langle\vec{b},\vec{y}\rangle-\langle\vec{a},\vec{y}\rangle\langle\vec{b},\vec{b}\rangle)x + \langle\vec{a},\vec{a}\rangle\langle\vec{b},\vec{y}\rangle-\langle\vec{a},\vec{y}\rangle\langle\vec{a},\vec{b}\rangle.
\end{align*}
Substituting $\vec{y}=\vec{a}+y\vec{b}$ into the above equation, we have 
\begin{align*}
    g(x)=&(\langle\vec{a},\vec{b}\rangle\langle\vec{b},\vec{a}+y\vec{b}\rangle-\langle\vec{a},\vec{a}+y\vec{b}\rangle\langle\vec{b},\vec{b}\rangle)x + \\
    &\langle\vec{a},\vec{a}\rangle\langle\vec{b},\vec{a}+y\vec{b}\rangle-\langle\vec{a},\vec{a}+y\vec{b}\rangle\langle\vec{a},\vec{b}\rangle\\
    =&[\langle\vec{a},\vec{b}\rangle(\langle\vec{a},\vec{b}\rangle+y\langle\vec{b},\vec{b}\rangle)-(\langle\vec{a},\vec{a}\rangle+y\langle\vec{a},\vec{b}\rangle)\langle\vec{b},\vec{b}\rangle]x +\\
    & [\langle\vec{a},\vec{a}\rangle(\langle\vec{a},\vec{b}\rangle+y\langle\vec{b},\vec{b}\rangle)-(\langle\vec{a},\vec{a}\rangle+y\langle\vec{a},\vec{b}\rangle)\langle\vec{a},\vec{b}\rangle]\\
    =&(||\langle\vec{a},\vec{b}\rangle||_2^2-||\vec{a}||_2^2||\vec{b}||_2^2)x+(||\vec{a}||_2^2||\vec{b}||_2^2-||\langle\vec{a},\vec{b}\rangle||_2^2)y\\
    =&(||\langle\vec{a},\vec{b}\rangle||_2^2-||\vec{a}||_2^2||\vec{b}||_2^2)(x-y).
\end{align*}
According to the Cauchy-Schwarz inequality, we can obtain 
\begin{align*}
    ||\langle\vec{a},\vec{b}\rangle||_2^2-||\vec{a}||_2^2||\vec{b}||_2^2\leq0
\end{align*}
Then, we have
\begin{equation*}
g(x)\left\{
\begin{aligned}
>0 & &x<y\\
=0 & &x=y \\
<0 & &x>y.
\end{aligned}
\right.    
\end{equation*}
thus, for any position vector $\vec{y}$, when $x=0$ or $x\rightarrow \infty$ ($\Vec{x} =\vec{a}$ or $\Vec{x} =\vec{b}$), the angle formed between $\Vec{y}$ and $\vec{x}$ is maximum.

Without loss of generality, given a specific $\vec{y}$, if its maximum vector $\vec{x}=\vec{a}$, we can then set $\vec{y}$ to $\vec{a}$ and find its maximum vector again, which yields
\begin{equation*}    \theta(\vec{x},\vec{y})\leq\theta(\vec{a},\vec{y})\leq\theta(\vec{b},\vec{a})
\end{equation*}
The proof is completed.
\end{proof}

\begin{theorem}
\label{theo:gen}
Consider the case of a convex hull that does not contain the origin. Let $X = \{x_1, x_2, \dots, x_k\}$ be a set of points and let $S$ be its convex hull. Then, the maximum angle between the position vectors of any two points in $S$ is formed by the position vectors of two extreme points of $S$.
\end{theorem}

\begin{proof}
Assume that this case holds when k.

When $k=2$, based on Lemma \ref{lemma:two}, the maximum angle is formed by the extreme points $\vec{x_1}$ and $\vec{x_2}$.

When $k\geq3$, we can sort the elements of X such that for a point $y$ in $S$, $\vec{x_k}$ maximizes the angle $\theta(\vec{y},\vec{x_k})$. Besides, the points $x$ in $S$ are of the form:
\begin{align*}
    &\lambda_1\vec{x_1}+\lambda_2\vec{x_2}+...+\lambda_k\vec{x_k}\\
    =&(\lambda_1+...+\lambda_{k-1})(\frac{\lambda_1 \vec{x_1}}{\lambda_1+...+\lambda_{k-1}}+...+\frac{\lambda_{k-1}\vec{x_{k-1}}}{\lambda_1+...+\lambda_{k-1}})\\
    &+\lambda_k\vec{x_k},
\end{align*}
where $(\frac{\lambda_1 \vec{x_1}}{\lambda_1+...+\lambda_{n-1}}+...+\frac{\lambda_{k-1}\vec{x_{k-1}}}{\lambda_1+...+\lambda_{k-1}})$ is a position vector of a point located within the convex hull induced by $\{x_1,x_2,...,x_{k-1}\}$. Through Lemma~\ref{lemma:two} and definition, we can obtain
\begin{align*}
    \theta(\vec{x},\vec{y})\leq\theta(\vec{x_k},\vec{y})
\end{align*}
For any two points x and y in a convex hull S, by setting $\vec{y}=\Vec{x_k}$ and using the above inequality twice, without loss of generality, we can assume that the vector $\vec{x_1}$ makes the largest angle with $\vec{x_k}$. Then, we can obtain
    \begin{align*}
    \theta(\vec{x},\vec{y})\leq\theta(\vec{x_k},\vec{y})\leq\theta(\vec{x_1},\vec{x_k})
\end{align*}

By definition, $\theta(\vec{x_1},\vec{x_k})$ is no greater than the maximum angle formed by any other two basis vectors.

The proof is completed.
\end{proof}

\begin{corollary}
When the convex hull of the input set $V$ does not contain the origin, the largest angle between any two features after self-attention $V' = SA(V)$ is always less than or equal to the largest angle between features in $V$.
\end{corollary}

\begin{remark}
In the Temporal Action Detection (TAD) task, the temporal feature sequences extracted by the pre-trained video classification backbone often exhibit high similarity and the pure Layer Normalization~\cite{ba2016layer} projects the input features onto the hyper-sphere in the high-dimensional space. Consequently, the convex hull induced by these features often does not encompass the origin. As a result, self-attention operation causes the input features to become more similar, reducing the distinction between temporal features and hindering the performance of the TAD task.
\end{remark}

\begin{figure*}[t]
    \centering{
    \includegraphics[width=0.8\linewidth]{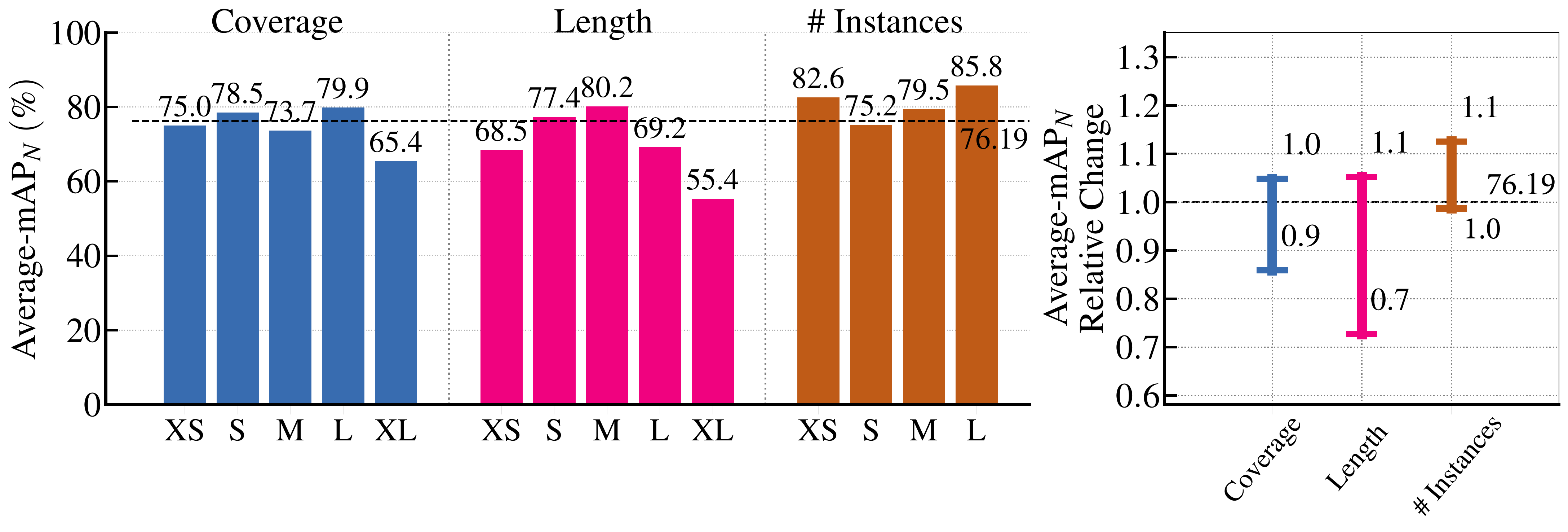}
    
  \caption{The sensitivity analysis of the detection results on THUMOS14. Where $mAP_{N}$ is the normalized mAP with the average number $N$ of ground truth segments per class~\cite{alwassel2018diagnosing}.}
  \label{fig:sensitivity} 
  }
\end{figure*}

\subsection{Error Analysis}
In this section, we analyse the detection results on THUMOS14 with the tool from ~\cite{alwassel2018diagnosing}, which analyze the results in three main directions: the False Positive (FP), the False Negative (FN) and the sensitivity of different length. For a further explanation of the analysis, please refer to ~\cite{alwassel2018diagnosing} for more details. 

\myPara{Sensitivity analysis.}
As shown in \figref{fig:sensitivity} (Left), three metrics are taken into consideration: coverage (the normalized length of the instance by the duration of the video), length (the actual length in seconds) and the number of instances (in a video). The results are divided into several length/number bins from extremely short (XS) to extremely long (XL). We can see that our method's performance is balanced over most of the action length, except for extremely long action instances which are significantly lower than the overall value (the dashed line). That's because extremely long action instances contain more complicated information, which deserves further exploration.

\begin{figure}[t!]
    \centering{
    \includegraphics[width=\linewidth]{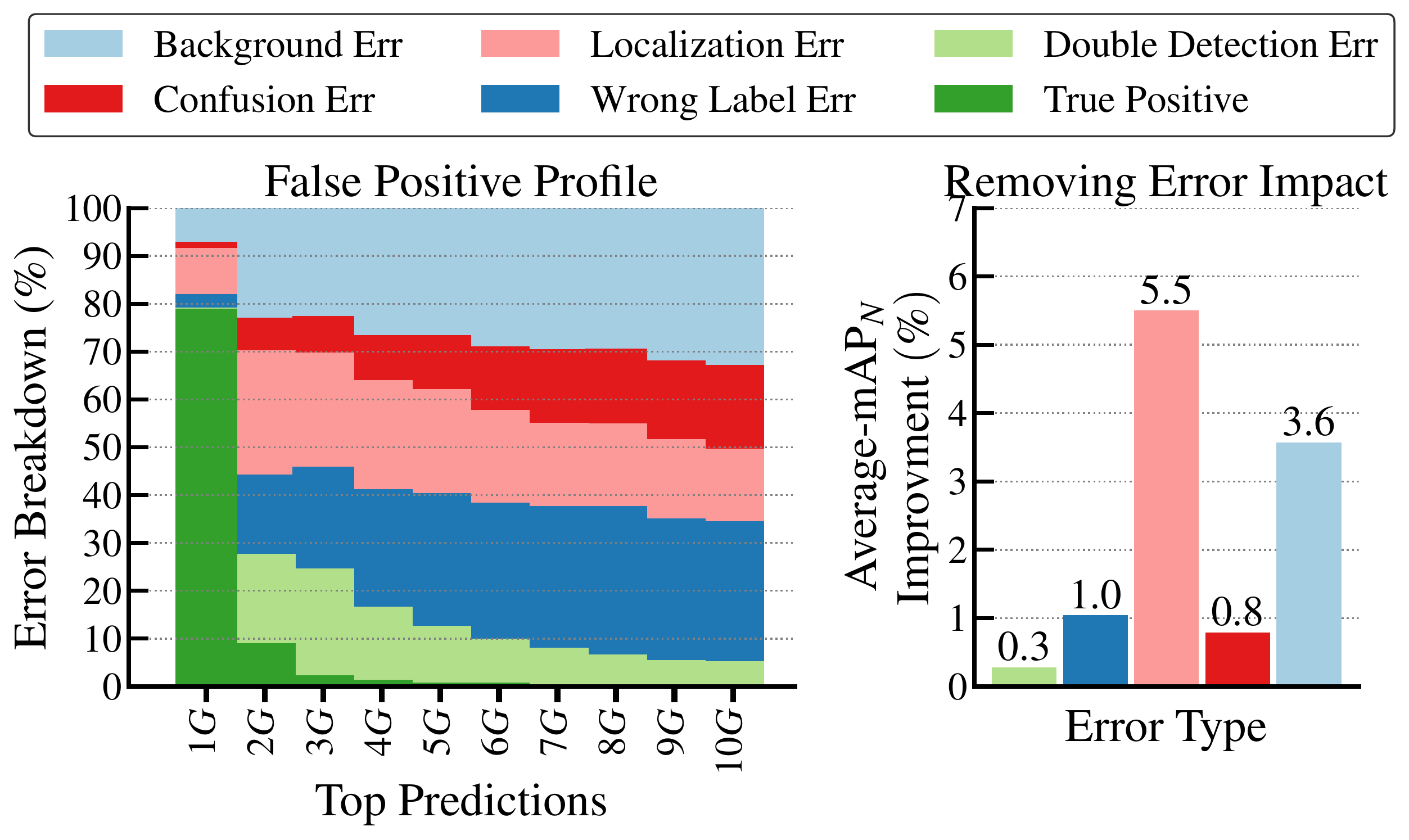}
    
  \caption{The false positive profile. It counts the percentage of several common types of detection error in different Top-K prediction groups.}
  \label{fig:fperror} 
  }
\end{figure}

\begin{figure}[t]
    \centering{
    \includegraphics[width=0.9\linewidth]{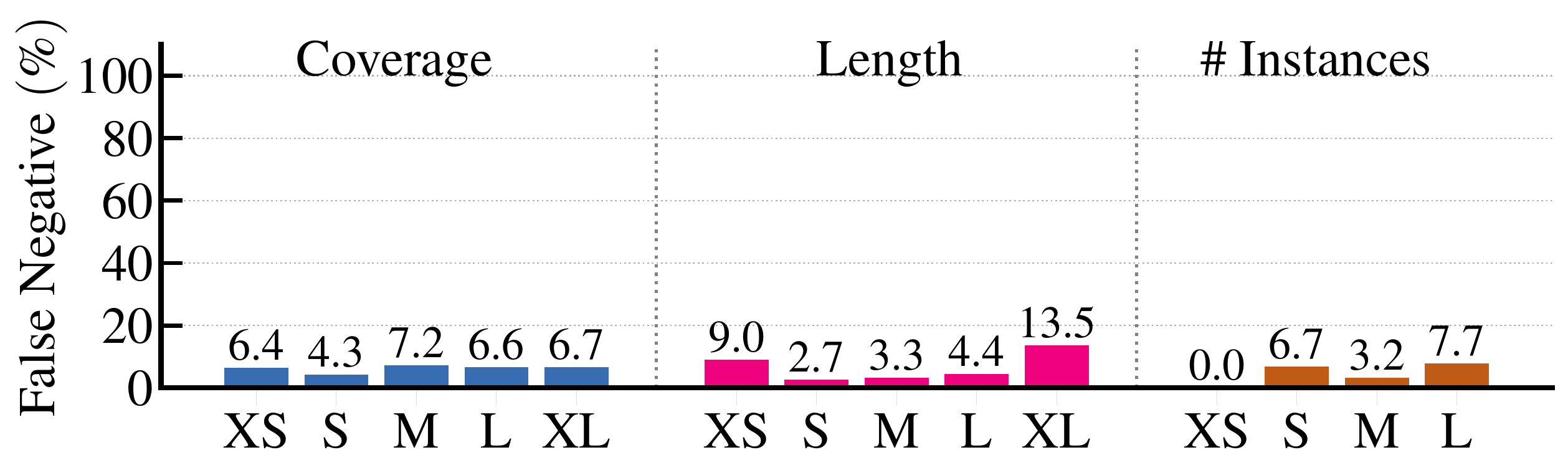}
    
  \caption{The false negative profile. It counts the percentage of miss-detection instances in different video lengths or videos with different action instance densities.}
  \label{fig:fnerror} 
  }
\end{figure}

\myPara{Analysis of the false positives.}
\figref{fig:fperror} shows a chart of the percentage of different types of action instances in different $k-G$ numbers, where G is the number of the ground-truth instances for each action category and the top $k \times G$ predicted instances are kept for visualization. 

From the $1G$ column on left, we can see in the top G prediction, the true positive instances account for about $80\%$ (at IoU=0.5), which indicates that our method has the power to estimate the right score of each instance. Moreover, on right, we can see the impact of each type of error: the regression error (\ie~ localization error and background error, the IoU between prediction and ground truth is much lower than a threshold or equal to zero) is still the part that deserves the most attention. 

\myPara{Analysis of the false negatives.}
In this section, we analyze the false negative (miss-detection) rate for our method. As depicted in \figref{fig:fnerror}, only the extremely short and extremely long instances have a relatively higher FN rate ($9.0\%$ and $13.5\%$, respectively), which is consistent with intuition that they are more difficult to detect. Note that for a video with only one action instance (XS), TriDet can detect all of them without any miss-detection ($0.0$ in \# Intances), demonstrating our advantage for single-action localization.

\begin{figure*}[t]
    \centering{
    \includegraphics[width=\linewidth]{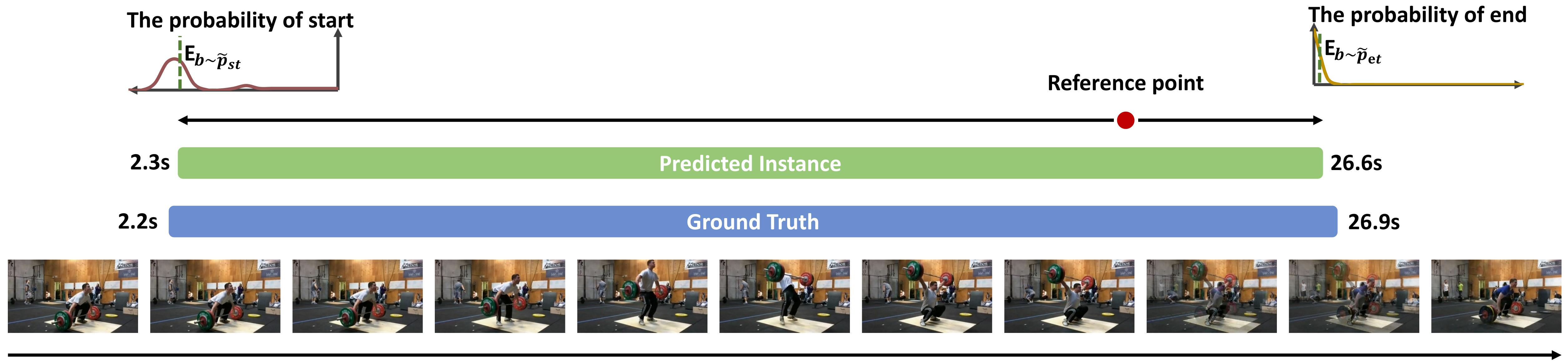}
  \caption{A visualization of the detection result on the THUMOS14 test set.}
  \label{fig:viz} 
  }
\end{figure*}

\subsection{Qualitative Analysis}
In \figref{fig:viz}, we show the visualization of a detection result on the THUMOS14 test set. It can be seen that our method accurately predicts the start and end instant of the action. Besides, we also visualize the predicted probability of the boundary in the Trident-head, where only the bin around the boundary has a relatively high probability while the others are low and smooth, indicating that the Trident-head can converge to a reasonable result.


\end{document}